\documentclass[letterpaper, 10 pt, conference]{ieeeconf} 
\IEEEoverridecommandlockouts

\usepackage{amsthm}
\usepackage{mathrsfs}
\usepackage{color}
\usepackage{amsfonts,amssymb}
\usepackage{amsfonts}
\usepackage{subfigure}
\usepackage{booktabs}
\usepackage{amsmath}
\usepackage{enumerate}
\usepackage[ruled,vlined,algo2e]{algorithm2e}
\SetKwRepeat{Do}{do}{while}%
\usepackage{algorithmic}
\usepackage{bm}
\usepackage{makecell}
\usepackage{hyperref}
\allowdisplaybreaks[4] 

\usepackage{mathtools}
\usepackage{bbm}
\usepackage{dsfont}
\usepackage{pifont}

\usepackage{cite}
\usepackage{graphicx}
\usepackage{balance}

\usepackage{setspace}

\newtheorem{remark}{Remark}
\newtheorem{theorem}{Theorem}
\newtheorem{lemma}{Lemma}

\newtheorem{assumption}{Assumption}
\newtheorem{problem}{Problem}
\newtheorem{definition}{Definition}
\newtheorem{proposition}{Proposition}

\usepackage{setspace}

\begin{document}

\title{{Model Selection for Inverse Reinforcement Learning\\ via Structural Risk Minimization}}
\author{Chendi Qu$^\dag$, Jianping He$^\dag$, Xiaoming Duan$^\dag$, and Jiming Chen$^\ddag$ 
	\thanks{
	 $^\dag$ the Dept. of Automation, Shanghai Jiao Tong University, Shanghai, China. E-mail address: \{qucd21, jphe, xduan\}@sjtu.edu.cn. $^\ddag$ the State Key Laboratory of Industrial Control Technology, Zhejiang University, Hangzhou, China. E-mail address: cjm@zju.edu.cn.
	}
}

\maketitle

\begin{abstract}%
Inverse reinforcement learning (IRL) usually assumes the reward function model is pre-specified as a weighted sum of features and estimates the weighting parameters only.
However, how to select features and determine a proper reward model is nontrivial and experience-dependent. A simplistic model is less likely to contain the ideal reward function, while a model with high complexity leads to substantial computation cost and potential overfitting.
This paper addresses this trade-off in the model selection for IRL problems by introducing the structural risk minimization (SRM) framework from statistical learning. SRM selects an optimal reward function class from a hypothesis set minimizing both estimation error and model complexity. To formulate an SRM scheme for IRL, we estimate the policy gradient from given demonstration as the empirical risk, and establish the upper bound of Rademacher complexity as the model penalty of hypothesis function classes. The SRM learning guarantee is further presented. In particular, we provide the explicit form for the linear weighted sum setting. Simulations demonstrate the performance and efficiency of our algorithm.

\end{abstract}

\section{Introduction}
The concept of learning from demonstration (LfD) has received substantial attention in recent years \cite{ravichandar2020recent}. LfD makes robots programming possible for non-experts and has been applied to various fields including autonomous driving \cite{kuderer2015learning}, manufacturing \cite{kent2016construction} and human-robot interaction \cite{maeda2017probabilistic}. Given the demonstration sampled from expert's outputs, one mainstream category of LfD algorithms is to learn the policy directly from the observations to action \cite{rahmatizadeh2018vision,torabi2018behavioral}, known as end-to-end learning. However, these methods usually require a large amount of data and do not generalize well. Another class named inverse reinforcement learning (IRL) conducts a two-stage algorithm, which is to infer the reward function first and then solve the forward problem with the learned objective to obtain the target policy \cite{arora2021survey, adams2022survey}. 
Since the objective is learned, IRL is able to imitate expert's policy as the environment and initial state change. 


Extensive effort went into solving the IRL problem \cite{abbeel2004apprenticeship, ramachandran2007bayesian, ziebart2010modeling, pirotta2016inverse, ho2016generative,ashwood2022dynamic}.
Standard IRL algorithms assume a predefined model for the reward function, which is usually described as a feature-based linear combination $r(s,a)=\sum_{i=1}^q \omega_i^T \phi_i(s,a) = \mathbf{\omega}^T \mathbf{\phi}(s,a)$, and only estimate the parameter $\mathbf{\omega}$. 
However, selecting a proper model (e.g. feature model $\phi_i(\cdot, \cdot)$) in IRL is a nontrivial problem. \cite{baimukashev2024automated} discusses a feature selection in IRL given polynomial form while they do not consider the model complexity and their features only involve the states without inputs. \cite{haug2018teaching} also considers the case when features are not consistent with the expert and decreases the difference by updating the learner's worldview. In this paper, we take the function model complexity into consideration and tackle a trade-off here. Notice that when the feature model is chosen to be too simple, the ideal reward function may not be contained in the class, resulting in a disparity between the learned policy and the authentic one due to the large approximation error. Conversely, if we search in a complex and rich function class, it requires more demonstration data to minimize the estimation error. And the computation can be extremely expensive, as many IRL algorithms involve a iterative weighting parameter learning and forward RL process \cite{ziebart2008maximum}. A complex reward function increases the computational cost in both IRL and forward RL.
Therefore, selecting an appropriate model for IRL is a crucial consideration.

Note that the previous described model selection problem is firstly proposed in the statistical
learning \cite{shawe1998structural,koltchinskii2001rademacher}. In a classification task, the choice of the hypothesis function class can be determined through solving a trade-off between the estimation error and approximation error. The estimation error is described by an empirical risk while the approximation error can be bounded by the Rademacher complexity \cite{mohri2018found} of the function class. Based on these, the structural risk minimization (SRM) is to find the optimal function class that minimizes the both two error terms. Recently, SRM scheme has been applied to solve traditional control problems. \cite{massucci2020structural} estimates the number of modes in a switched system with SRM. \cite{stamouli2023structural} tackles SRM problems for nonlinear system identification over hierarchies of model class including norm-constrained reproducing kernel Hilbert space and neural network.

Motivated by the above discussion, we introduce the SRM scheme to solve the model selection for IRL with unknown features. The challenge lies in establishing the empirical risk and the corresponding upper bound for the complexity of the hypothesis model.
The main contributions are summarized as follows.
\begin{itemize}
\item We leverage the estimated policy gradient from expert's demonstration trajectories to derive the empirical risk in IRL. And we establish the upper bound complexity measure on this gradient-based risk given the the Rademacher complexity of hypothesis function models. 
\item  We propose an SRM scheme for IRL by minimizing a weighted combination of the empirical risk and the complexity bound, which determines the optimal model for the reward and learns its corresponding parameters. The learning guarantee of the SRM solution is provided. Particularly, we present an explicit complexity bound and algorithm flow for the common reward setting of a linear weighted sum of features. 
\item  Numerical simulations on a linear quadratic regulator (LQR) control are conducted to show the efficiency and performance of our SRM scheme.
\end{itemize}

The remainder of the paper is organized as follows. Section \ref{ppf_sec} describes the problem of interest and introduces related preliminaries. Section \ref{srm_sec} presents the SRM scheme design for IRL with unknwon reward model. Numerical experiments and simulation results are shown
in Section \ref{sim_sec}, followed by the conclusion in Section \ref{conc_sec}.

\section{Preliminaries and Problem of Interest}\label{ppf_sec}
\subsection{IRL with Unknown Reward Function Model}

Consider a Markov decision process (MDP) defined by a tuple $(\mathcal{S}, \mathcal{A}, p, \gamma, r)$, where $\mathcal{S} \in \mathbb{R}^{|\mathcal{S}|}$ is the state space, $\mathcal{A}\in \mathbb{R}^{|\mathcal{A}|}$ is the action space. The environment dynamics are characterized by the state transition model $p$, where $p(s'|s,a) \in \{0,1\}$ denotes the probability of the transition from state $s$ to $s'$ under action $a$. $\gamma \in [0,1)$ is the discount factor and $r : (s,a) \mapsto \mathbb{R}$ is the reward function. 

The goal of IRL is to identify the objective function $r$ based on a demonstration $\mathcal{T}$ generated by an expert following the optimal policy $\pi: \mathcal{S} \xrightarrow{} \mathcal{P}(\mathcal{A})$.
Different from  standard IRL studies where the reward function model is pre-specified, this paper considers a more practical scenario, where the model and parameters both are unknown. 
The problem is described as follows.
\begin{problem}\label{main_pro}
Suppose we have a demonstration $\mathcal{T}$ and a finite countable set of hypothesis reward function classes $\{\mathcal{F}_j\}_{j=1}^C$. IRL with unknown reward model is to select an optimal class index $j^*$ and identify the optimal reward function $r^* \in \mathcal{F}_{j^*}$.
\end{problem}

However, properly defining the optimality in Problem \ref{main_pro} is nontrivial.
In a rich function class, we are more likely to find an ideal function that explains the objective well, while searching in such a complex space leads to a large computation cost and requires more data. Therefore, to obtain an optimal model balancing this trade-off, we introduce the SRM scheme in statistical learning into our problem. 
We first provide the following two preliminaries.

\subsection{Policy Gradient Minimization}

The reward function of a $\gamma$-discounted reinforcement learning task is defined as
\begin{equation}
J(\pi,r)= \int_{\mathcal{S}} d^{\pi}(s) \int_{\mathcal{A}} \pi(a;s,\theta) Q^{\pi}(s,a) {d}a {d}s,
\end{equation}
where $d^{\pi}(s)$ is the stationary distribution of state $s$ under policy $\pi$, and $Q^{\pi}(s,a) = \mathbb{E}\{\sum_{k=1}^\infty \gamma^{k-1} r_{t+k} | s_t=s,a_t=a, \pi\}$ \cite{sutton1999policy}. In policy gradient approaches \cite{grondman2012survey}, the policy is required to be stochastic, which can be achieved by adding zero-mean Gaussian noise, or we need to obtain the transition model $p$. Assuming $\pi$ is differentiable with respect to its parameter $\theta$, we can calculate the gradient of the objective function with respect to $\theta$ as
\begin{equation}\label{grad}
\frac{\partial J(\pi,r)}{\partial \theta} = \int_{\mathcal{S}} d^{\pi}(s) \int_{\mathcal{A}} \frac{\partial \pi(a;s,\theta)}{\partial \theta} Q^{\pi}(s,a) {d}a {d}s.
\end{equation}
For the gradient in \eqref{grad}, we have the following proposition.
\begin{proposition}[Optimality Necessary Condition]\label{zero_grad}
If the expert policy $\pi^*$ is the optimal policy with respect to the designed reward $r(s,a)$, the gradient $\frac{\partial J(\pi^*,r)}{\partial \theta}$ will equal to $\textbf{0}$, which means $\pi^*$ is a stationary point of $J(\pi,\omega)$.
\end{proposition}

\subsection{Rademacher Complexity}
Rademacher complexity  measures the capacity of a hypothesis class of real-valued functions. Denote $\mathcal{X}, \mathcal{Y}$ as the input and output spaces in a regression problem. $\mathcal{F}$ is a hypothesis function class, where $f:\mathcal{X} \xrightarrow{} \mathcal{Y}, f\in \mathcal{F}$. For an arbitrary loss function class $\mathcal{L}^{reg}$ associated with $\mathcal{F}$ mapping from $\mathcal{Z} = \mathcal{X} \times \mathcal{Y}$ to $\mathbb{R}$, we have
\[
\mathcal{L}^{reg}\!=\!\{l^{reg}(z): (x,y) \mapsto l^{reg}(f(x),y), z:=(x,y), f\in \mathcal{F}\}.
\]
We then provide the following definition.

\begin{definition}[Rademacher Complexity \cite{bartlett2002rademacher}]\label{rade}
Given $\mathcal{L}^{reg}$ as a function class mapping from $\mathcal{Z}$ to $\mathbb{R}$, and $S=\{z_i\}_{i=1}^m$ as a sequence of $m$ samples from $\mathcal{Z}$, the empirical Rademacher complexity of $\mathcal{L}^{reg}$ with respect to $S$ is defined as
\begin{equation}
\hat{\mathfrak{R}}_S(\mathcal{L}^{reg}) = \mathbb{E}_{\sigma} \left[ \sup_{l^{reg}\in \mathcal{L}^{reg}} \frac{1}{m} \sum_{i=1}^m \sigma_i l^{reg}(z_i) \right],
\end{equation}
where $\sigma = (\sigma_i)_{i=1}^m$ are independent uniform random variables distributed in $\{-1,1\}$. $\sigma_i$ is called Rademacher variable. The Rademacher complexity of $\mathcal{L}^{reg}$ is
$\mathfrak{R}_m(\mathcal{L}^{reg}) = \mathbb{E}_{S} \hat{\mathfrak{R}}_S(\mathcal{L}^{reg})$.
\end{definition}

The following theorem shows that Rademacher complexity provides an upper bound on the expected risk (also generalization error). If the complexity of the model $\mathcal{L}^{reg}$ is high, the generalization error may be large and it is more difficult to generalize to unseen data, making it prone to overfitting. As the sample number $m$ increases, the bound becomes tighter.
\begin{theorem}[Theorem 3.3 in \cite{mohri2018found}]\label{rade_def}
Let $\mathcal{L}^{reg}$ be a function class mapping from $\mathcal{Z}$ to $[0,B]$ and $S=\{z_i\}_{i=1}^m$ be a sequence of $m$ i.i.d. samples from $\mathcal{Z}$. Then for any $\delta>0$, with at least $1-\delta$ probability
\begin{equation}
\mathbb{E}[l^{reg}(z)] \leq \frac{1}{m} \sum_{i=1}^m l^{reg}(z_i) + 2 \hat{\mathfrak{R}}_S(\mathcal{L}^{reg}) + 3 B \sqrt{\frac{\log \frac{2}{\delta}}{2m}}
\end{equation}
holds for all $l^{reg}\in \mathcal{L}^{reg}$.
\end{theorem}
Noticing Theorem \ref{rade_def} requires $\mathcal{L}^{reg}$ to be bounded, we define the clipped version $\bar{f}$ of a function $f$ as
\begin{equation}\label{clip}
\bar{f}(x) = \left\{
\begin{array}{ll}
f(x) \cdot \frac{B}{\Vert f(x) \Vert_2}, ~ \textup{if} ~\Vert f(x) \Vert_2>B\\
f(x), ~\textup{otherwise}
\end{array}
\right.
.
\end{equation}

\section{SRM Scheme for IRL with Unknown Reward Model}\label{srm_sec}
Based on preliminaries introduced in the previous section, we now derive the SRM method for Problem \ref{main_pro}. We will first build the empirical risk minimization (ERM) based on the policy gradient estimation and the upper bound of model complexity with Rademacher complexity. Then we provide the SRM scheme, which consists of selecting an optimal reward function class index $j^* \geqslant 1$ and the ERM hypothesis $r^* \in \mathcal{F}_{j^*}$, achieving a minimization on both estimation error and the model complexity.

\subsection{ERM-IRL based on Policy Gradient}

Suppose there is an expert following the optimal policy $\pi$ with respect to a set reward function $r$. We have observed a trajectory $\tau = (s_0,a_0,s_1,\dots,s_T)$, where $s_0 \sim \mathcal{D}$ is the initial state generated from a distribution $\mathcal{D}$ on space $\mathcal{S}$.
Then the objective gradient \eqref{grad} with $\pi,r$ can be estimated through some existing methods such as REINFORCE \cite{williams1992simple}, GPOMDP \cite{baxter2001infinite} and eNAC \cite{peters2008natural}. For the simplicity of the calculation and formulation, we utilize REINFORCE and the gradient is calculated as
\begin{equation}\label{es_grad}
\nabla_{\theta} \hat{J}_{\tau}(s_0,r) = \sum_{t=1}^{T} \nabla_{\theta} \ln{\pi_{t}} \sum_{k=t}^{T} \gamma^{k-t} r(s_{k}(s_0),a_{k}(s_0)),
\end{equation}
where $\nabla_{\theta} \ln{\pi_{t}} = \frac{\partial \ln{\pi(a_t;s_t,\theta)}}{\partial \theta}$ for brevity. The policy $\pi(a_t;s_t,\theta)$ at time $t$ in trajectory $\tau$ only depends on state $s_0$. Thus the gradient $\nabla_{\theta} \hat{J}_{\tau}$ in \ref{es_grad} is a function of $s_0,r$.

When the sampling on $\mathcal{D}$ is infinite and stochastic, the estimated gradient \eqref{es_grad} goes to zero according to Proposition \ref{zero_grad}. For an IRL problem, the real reward function $r$ can be identified by minimizing the following criterion
\begin{equation}\label{real-err}
\epsilon_{\mathcal{D}}(r) := \mathbb{E}_{s_0 \sim \mathcal{D}} l( \nabla_{\theta} \hat{J}_{\tau}(s_0,r)),
\end{equation}
where $l(\cdot)$ is a loss function on the gradient $\nabla_{\theta} \hat{J}_{\tau}$ satisfying Assumption \ref{l_ass}.

\begin{assumption}\label{l_ass}
The loss function $l(\cdot)$ is Lipschitz continuous with $l(0) = 0$.
\end{assumption}
However, in the realistic situation, the distribution $\mathcal{D}$ is unknown and we can only observe a set of $M$ trajectories $\mathcal{T}=\{\tau^i\}_{i=1}^M$ with a set $\{s_0^i\}_{i=1}^M$ of initial states. Thus, we use the empirical risk for IRL problem as an estimation to \eqref{real-err} which is described by
\begin{equation}\label{erm_risk}
\begin{aligned}
\hat{\epsilon}_{\mathcal{T}} (r) := \frac{1}{M} \sum_{i=1}^M l(\nabla_{\theta} \hat{J}_{\tau^i}(s_0^i,r)).
\end{aligned}
\end{equation}
\begin{problem}[ERM-IRL Problem]\label{erm-irl}
Given an observed demonstration $\mathcal{T}$ and a function class ${\mathcal{F}}$, the ERM hypothesis reward function is obtained through
\begin{equation}\label{erm_pro}
r^{ERM}_{\mathcal{T}} = \mathop{\arg\min}\limits_{r \in \mathcal{F}} \hat{\epsilon}_{\mathcal{T}} (r).
\end{equation}
\end{problem}

\begin{assumption}
During the gradient estimation, we assume the function model of the expert policy $\pi(a;s,\theta)$ is known or pre-selected properly. The parameter $\theta$ (e.g. the feedback gain matrix $K$ in LQR) can be estimated through a maximum likelihood problem based on $\mathcal{T}$.
\end{assumption}

\subsection{SRM Scheme for IRL}
Note that ERM in Problem \ref{erm-irl} cannot deal with the trade-off regarding to the model complexity of the chosen function class $\mathcal{F}$. In this subsection, we first build the upper bound on approximation error with Rademacher complexity of $\mathcal{F}$.

For a reward function class $\mathcal{F}$ and a trajectory sample $\tau$ with $s_0$, a risk function class is defined as 
\[\mathcal{L}^{irl} = \{l^{irl}:s_0 \mapsto l(\nabla_{\theta} \hat{J}_{\tau}(s_0,r)), r\in \mathcal{F}\}\] 
based on the established risk \eqref{real-err}. The following lemma derives an upper bound on the Rademacher complexity of $\mathcal{L}^{irl}$ with that of $\mathcal{F}$.
\begin{lemma}\label{rade_l}
Given a demonstration $\mathcal{T}=\{\tau^i\}_{i=1}^M$ with a set $\{s_0^i\}_{i=1}^M$ of initial states and the reward function class $\mathcal{F}$, the Rademacher complexity of $\mathcal{L}^{irl}$ is bounded by
\begin{equation}
\hat{\mathfrak{R}}_{\mathcal{T}}(\mathcal{L}^{irl}) \leqslant L \sum_{t=1}^{T} \sum_{k=t}^{T} \gamma^{k-t} \nabla_{\theta} \Tilde{\pi}_t \cdot \hat{\mathfrak{R}}_{\mathcal{T}_k}(\mathcal{F}) := R_{\mathcal{T}}(\mathcal{F}),
\end{equation}
where $\mathcal{T}_k = \{(s^i_k,a^i_k)\}_{i=1}^M,k=1,\dots,T$ and $\nabla_{\theta} \Tilde{\pi}_t = \max_{i} \Vert \nabla_{\theta} \ln{\pi^i_{t}} \Vert$. $L$ is the Lipschitz constant of the loss function $l(\cdot)$.
\end{lemma}
\begin{proof}
See the proof in Appendix \ref{app_a}.
\end{proof}

Then, through a weighted sum of empirical risk \eqref{erm_risk} and the upper bound $R_{\mathcal{T}}(\mathcal{F})$ on complexity, we define the SRM problem as follows.

\begin{problem}[SRM-IRL Problem]
Given an observed demonstration $\mathcal{T}$ and a series of hypothesis reward function class $\{\mathcal{F}_j\}_{j=1}^C$, the optimal SRM solution for IRL is defined as
\begin{equation}\label{srm_eq}
r_{\mathcal{T}}^{SRM} = \mathop{\arg\min}_{1\leqslant j \leqslant C, r\in \mathcal{F}_j} J_{\mathcal{T}}(r,j),
\end{equation}
where
\[
J_{\mathcal{T}}(r,j) := \hat{\epsilon}_{\mathcal{T}} (r) + 
2 R_{\mathcal{T}}(\mathcal{F}_j).
\]
\end{problem}

\begin{theorem}\label{srm_the}
Given an observed demonstration $\mathcal{T}$ generated from $\mathcal{D}$ and a set of hypothesis reward function classes $\{\mathcal{F}_j\}_{j=1}^C$, for any $r \in \mathcal{F}_j$, we define the clipped version $\bar{r}$ and $\Vert\bar{r}\Vert\leqslant B$. Then, for $\delta \in (0,1]$, we provide the following error bounds in IRL model selection.\\
(i) Union bound: For all $\bar{r} \in \mathcal{\bar{F}}_j, 1\leqslant j \leqslant C$, there is at least $1-\delta$
probability such that 
\begin{equation}\label{app-bound}
\begin{aligned}
\epsilon_{\mathcal{D}}(\bar{r}) \leqslant & \hat{\epsilon}_{\mathcal{T}}(\bar{r}) + 2 R_{\mathcal{T}}(\bar{\mathcal{F}}_j) +\\ &3 LB \sum_{t=1}^{T} \sum_{k=t}^{T} \gamma^{k-t}  \Vert \nabla_{\theta} \ln{\pi_{t}}\Vert \sqrt{\frac{\log \frac{4}{\delta}}{2M}}.
\end{aligned}
\end{equation}
(ii) SRM learning bound: For the SRM solution defined by \eqref{srm_eq}, with the probability $1-\delta$,
\begin{equation}
\begin{aligned}
\epsilon_{\mathcal{D}}(\bar{r}_{\mathcal{T}}^{SRM}) \leqslant & \min_{r \in \mathcal{F}} \left( \epsilon_{\mathcal{D}}(\bar{r}) + 4 R_{\mathcal{T}}(\bar{\mathcal{F}}_{j(r)}) \right)+ \\
& 3 LB \sum_{t=1}^{T}  \sum_{k=t}^{T} \gamma^{k-t} \Vert \nabla_{\theta} \ln{\pi_{t}} \Vert \sqrt{\frac{\log \frac{4(C+1)}{\delta}}{2M}}
\end{aligned}
\end{equation}
holds, where $j(r)$ refers to the index of a hypothesis function class that $r \in \mathcal{F}_{j(r)}$. 
\end{theorem}
\begin{proof}
See the proof in Appendix \ref{app_b}.
\end{proof}
The SRM learning bound in Theorem \ref{srm_the} demonstrates an optimality guarantee for the SRM hypothesis solution. 
The general algorithm is shown in Algorithm \ref{alg}. Note that to achieve the optimality under structural risk, we need to go through the candidate function models, while this extra computation cost is acceptable if we see this model selection as a pre-process for IRL. Especially when IRL algorithm involving iterative learning, an improper reward function can lead to much higher complexity and larger cost subsequently.
\begin{algorithm2e}\label{alg}
\SetAlgoLined
 \caption{SRM Algorithm for IRL with Unknown Reward Model}
    \KwIn{ 
    The demonstration from expert, $\mathcal{T}$; The hypothesis function classes, $\{\mathcal{F}_j\}_{j=1}^C$;}
    \KwOut{
      The optimal class index $j^*$; The reward function identification $r_{\mathcal{T}}^{SRM} \in \mathcal{F}_{j^*}$;}
    \For{$j=1$ to $C$}{Solve ERM problem \eqref{erm_pro} to obtain the optimal estimation $r_j^*$ in function class $\mathcal{F}_j$;\\
    Calculate the structural risk $J_{\mathcal{T}}^{SRM}(r_j^*,j)$ in \eqref{srm_eq};}
    Select the minimum risk and its corresponding class number $j^* = \mathop{\arg\min}_{1\leqslant j \leqslant C} J_{\mathcal{T}}^{SRM}(r_j^*,j) $;\\
\textbf{return} The optimal reward function estimation $r_{\mathcal{T}}^{SRM} = r^*_{j^*}$.
\end{algorithm2e}

\subsection{Linear Weighted Sum Case}
In this subsection, we treat the special linearly weighted feature-based sum reward function described as:
\begin{equation}\label{linear_sum_r}
r(s,a;\omega)=\sum_{p=1}^q \omega_p^T \phi_p(s,a),
\end{equation}
which is parameterized by $\omega =(\omega^T_1, \dots, \omega^T_q)^T$ with features $\phi_p$ mapping from $\mathcal{S} \times \mathcal{A}$ to $\mathbb{R}^{n_p}$. This is a quite general setting. Taking the classic LQR problem as an example, suppose the reward for $(s,a)$ pair is $r = - \sum_{p=1}^q (s^T Q_p s + a^T R_p a)$. Then, write it into the feature-based sum form. We have $\omega_p = (vec(Q)^T, vec(R)^T)$ and feature $\phi_p(s,a)=((s \otimes s)^T, (a \otimes a)^T)^T$. Traditional IRL is to identify the parameter $\omega$ with the given demonstration, assuming the features $\{\phi_p\}_{p=1}^q$ are known. 

We now utilize our proposed SRM scheme to analyze the IRL under this linear combination setting with unknown features. Supposing $C$ hypothesis reward function classes $\{\mathcal{F}_j\}_{j=1}^C$ with different sets of features $\{\phi_p^j\}_{p=1}^q$, we are going to obtain the optimal function class $\mathcal{F}_{j^*}$ and its corresponding optimal parameter estimation $\hat{\omega}^*$.
With this specific reward function form, we provide the following lemma establishing the explicit bound of Rademacher complexity bound for \eqref{linear_sum_r}.
\begin{lemma}\label{linear_rade} 
For the linear weighted sum case, assuming the parameter $\Vert \omega_p \Vert \leqslant B_{\omega}$ for all $p$ and $\Phi_p(k)$ is the bound for $\Vert \phi_p(s,a) \Vert$ on dataset $\mathcal{T}_k, k=1,\dots,T$, then we have
\begin{equation}
\hat{\mathfrak{R}}_{\mathcal{T}_k}(\mathcal{F}) \leqslant \frac{B_{\omega} }{\sqrt{M}} \sum_{p=1}^q \Phi_p(k).
\end{equation}
\end{lemma}
\begin{proof}
See the proof in Appendix \ref{app_c}.
\end{proof}

Note that when considering different function class $\mathcal{F}_j$, there exist different minimum upper bounds $B^j_{\omega}$ for $\Vert \omega^j_p \Vert$. The following remark provides a uniform bound for all $j=1,\dots, C$.

\begin{remark}\label{omega_rem}
When considering the linear weighted form reward, the expected policy gradient equation ${\epsilon}_{\mathcal{D}}(r) = 0$ has a scalar ambiguity property, which suggests that all $\omega' = \alpha \omega, \alpha \in \mathbb{R}_{+}$ satisfy the equation. To eliminate this ambiguity and avoid the trivial zero solution during the ERM, we add a unit simplex constraint 
\begin{equation}\label{omega_cons}
\{\omega \geqslant 0, \, \Vert \omega \Vert_1 = 1\}.
\end{equation}
     Coincidentally, this constraint makes $\omega_p$ and reward $r$ bounded, which just suits the Rademacher complexity definition. Therefore, in this case we do not need to do the function clipping in \eqref{clip} and a uniform bound for all function classe $\{\mathcal{F}_j\}_{j=1}^C$ is $\hat{\mathfrak{R}}_{\mathcal{T}_k}(\mathcal{F}_j) \leqslant \frac{1 }{\sqrt{M}} \sum_{p=1}^q \Phi^j_p(k)$.
\end{remark}

Based on Lemma \ref{linear_rade} and the constraints in Remark \ref{omega_rem}, the SRM-IRL problem for linear weighted sum reward is described as
\begin{equation}\label{srm_linear}
\begin{aligned}
r_{\mathcal{T}}^{SRM} := & \mathop{\arg\min}_{1\leqslant j \leqslant C, r\in \mathcal{F}_j} \left( \hat{\epsilon}_{\mathcal{T}} (r) + \right. \\
& \left. 2 L \sum_{t=1}^{T} \sum_{k=t}^{T} \gamma^{k-t} \nabla_{\theta} \Tilde{\pi}_t \cdot \frac{\sum_{p=1}^{q_j} \Phi_p(k)}{\sqrt{M}} \right).
\end{aligned}
\end{equation}
Combined with Theorem \ref{srm_the}, we provide the SRM learning bound in linear weighted sum case.

\begin{theorem}
For the SRM solution defined by \eqref{srm_linear}, with at least $1-\delta$ probability,
\begin{equation}
\begin{aligned}
& \epsilon_{\mathcal{D}}({r}_{\mathcal{T}}^{SRM}) \leqslant  \min_{r \in \mathcal{F}} \left( \epsilon_{\mathcal{D}}({r}) + 4 L \sum_{t=1}^{T} \sum_{k=t}^{T} \gamma^{k-t} \nabla_{\theta} \Tilde{\pi}_t \cdot \frac{1 }{\sqrt{M}} \right. \\
& \left.  \sum_{p=1}^q \Phi^{j(r)}_p(k) \right) + 3 LB \sum_{t=1}^{T}  \sum_{k=t}^{T} \gamma^{k-t} \Vert \nabla_{\theta} \ln{\pi_{t}} \Vert \sqrt{\frac{\log \frac{4(C+1)}{\delta}}{2M}}
\end{aligned}
\end{equation}
holds, where $j(r)$ refers to the index of a hypothesis function class that $r \in \mathcal{F}_{j(r)}$.
\end{theorem}

\subsection{Candidate Function Models}
The candidate function classes for SRM-IRL framework are suggested to be nested, i.e., $\mathcal{F}_j \subset \mathcal{F}_{j+1}$ for $\forall j$ though our algorithm can deal with arbitrary candidate classes. One choice for the hierarchical function set is polynomial basis. We can have
\[
\mathcal{F}_j = \sum_{p=1}^j \omega_p \begin{pmatrix}
s^{\otimes p} \\ a^{\otimes p}
\end{pmatrix}
\]
where $s^{\otimes p} := s \otimes s \cdots \otimes s$ for total $p$ times. Another alternative option is to define function classes in the Gaussian Reproducing Kernel Hilbert Space (RKHS) to achieve a rich and flexible representation. For a Gaussian kernel $\kappa_{\sigma}(x,x') = \exp(-\sigma \Vert x-x'\Vert_2^2)$, define a function class with the restriction on its corresponding RKHS $\mathcal{H}_{\sigma}$, which is
\[
\mathcal{F}_j = \{f\in \mathcal{H}_{\sigma_j}: \Vert f \Vert_{\mathcal{H}_{\sigma_j}} \leqslant B_j\},
\]
where $B_j >0$ is a constant controlling the complexity. We can set the parameters $(\sigma_j,B_j)$ to satisfy $\mathcal{F}_j \subset \mathcal{F}_{j+1}$.

\section{Simulation Results}\label{sim_sec}
In this section we conduct IRL simulations on the LQR problem to illustrate the performance of the proposed SRM scheme. Consider a dynamic system
\[
s_{t+1} = A s_t + B a_t, ~ t=0,1,\dots,T-1,
\]
where $s_t,a_t \in \mathbb{R}^n$, $A,B$ are dynamics matrices satisfying the condition of controllability. The control policy $\pi$ is $a_t = k s_t$. We set the reward function as
\begin{equation}\label{sim_lqr}
r_t = - s_t^T Q s_t + a_t^T R a_t=- \omega^T \begin{pmatrix}
s \otimes s \\ a \otimes a
\end{pmatrix},
\end{equation}
where the weighting matrices $Q,R$ are positive definite. Let $T=50$ and discount $\gamma = 0.9$.

For the data collection, we first train the policy through REINFORCE with baseline shown in Fig. \ref{train_pic}. We collect $M$ trajectories after $60$ episodes with random initial states generated from a uniform distribution. Using these data as demonstration $\mathcal{T}$, ERM-IRL in Problem \ref{erm-irl} is to estimate the weighting parameter $\omega$ when the feature function model is known (as \eqref{sim_lqr}). Fig. \ref{erm_pic} illustrates the estimation result with respect to the dataset size (trajectory number). As the size of the dataset gets larger, the estimation error $\Vert \hat{\omega} - \omega \Vert$ decreases.

\begin{figure}[h]
	\centering
	\begin{minipage}[c]{0.235\textwidth}
		\centering
        \includegraphics[width=\textwidth]{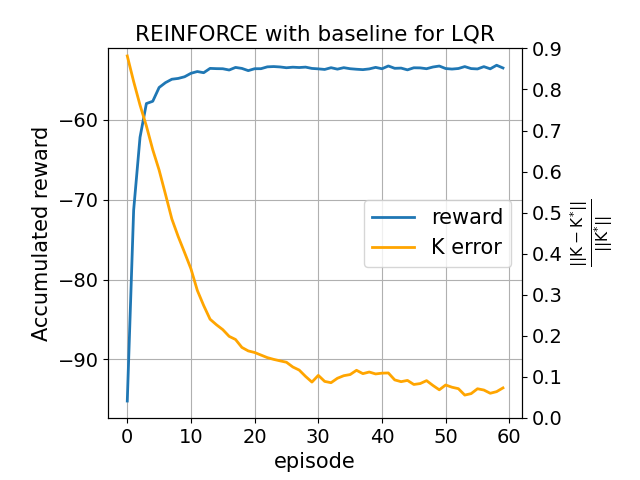}
		\caption{Policy gradient optimization for LQR. The demonstration is collected after 60 episodes.
  }
		\label{train_pic}
	\end{minipage} 
	\begin{minipage}[c]{0.235\textwidth}
		\centering
	\vspace{-15pt}	\includegraphics[width=\textwidth]{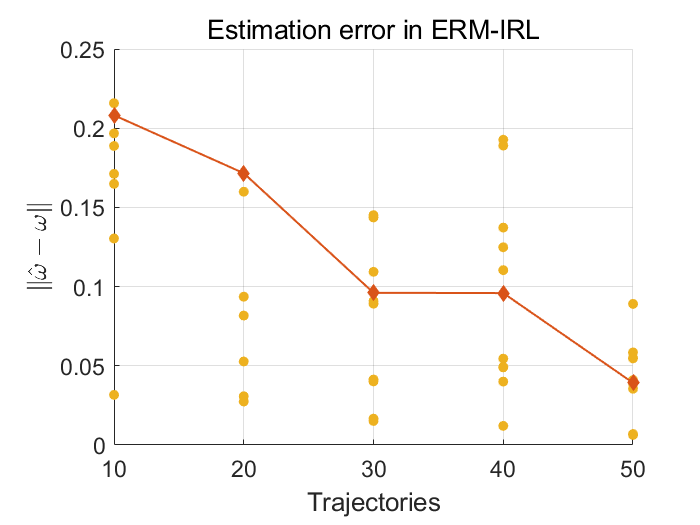}
	\caption{ERM-IRL estimation error.}
		\label{erm_pic}
	\end{minipage}
\end{figure}
To demonstrate the proposed SRM scheme, we choose $C=5$ hypothesis function classes based on polynomial basis listed as:
\begin{align*}
& \mathcal{F}_1 = \{\begin{pmatrix}
s \\ a
\end{pmatrix}\}, ~\mathcal{F}_2 = \{\mathcal{F}_1 \cup \begin{pmatrix}
s^{\otimes 2} \\ 0
\end{pmatrix}\}, ~ \mathcal{F}_3 = \{\mathcal{F}_1 \cup \begin{pmatrix}
s^{\otimes 2} \\ a^{\otimes 2}
\end{pmatrix}\}\\
& \mathcal{F}_4 = \{\mathcal{F}_3 \cup \begin{pmatrix}
s^{\otimes 3} \\ 0
\end{pmatrix}\}, ~\mathcal{F}_5 = \{\mathcal{F}_3 \cup \begin{pmatrix}
s^{\otimes 3} \\ a^{\otimes 3}
\end{pmatrix}\}.
\end{align*}
\noindent Note that the above hypothesis classes $\{\mathcal{F}_j\}_{j=1}^5$ are nested, i.e. $\mathcal{F}_j \subset \mathcal{F}_{j+1}$. We have ${\mathfrak{R}}_{M}(\mathcal{F}_j) < {\mathfrak{R}}_{M}(\mathcal{F}_{j+1})$ according to the monotonicity property of Rademacher complexity. The loss function $l(\cdot)$ is set to be the 2-norm of the gradient. We use $M=1000$ trajectories as the demonstration $\mathcal{T}$ to run the explicit optimization problem \eqref{srm_linear}.

\begin{figure}[ht]
\setlength{\abovecaptionskip}{0.1cm}
  \centering 
    \subfigure[Empirical risk and complexity penalty]{ 
    \label{lam1} 
    \includegraphics[height=0.25\textwidth,width=0.31\textwidth]{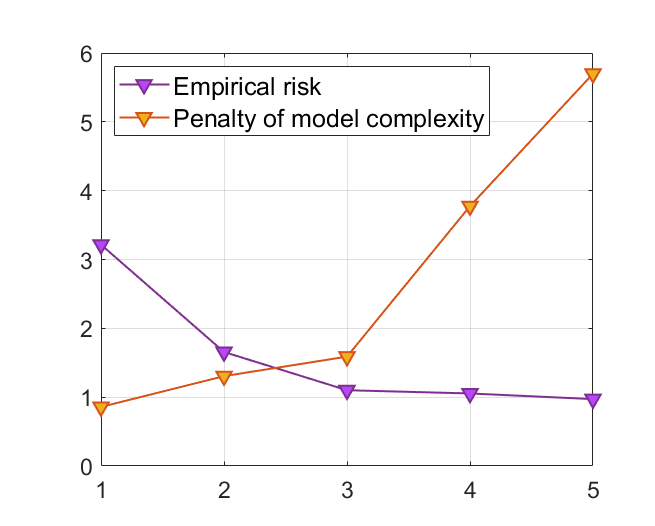} 
    }
    \subfigure[Structural risk (sum of two terms)]{ 
    \label{lam2} 
    \includegraphics[height=0.25\textwidth,width=0.31\textwidth]{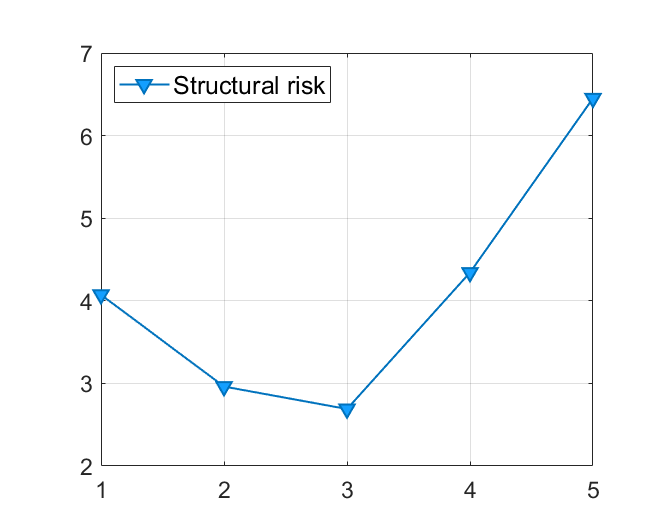} 
    }
    \subfigure[Statistics of 50 trials]{ 
    \label{lam3} 
    \includegraphics[height=0.25\textwidth,width=0.31\textwidth]{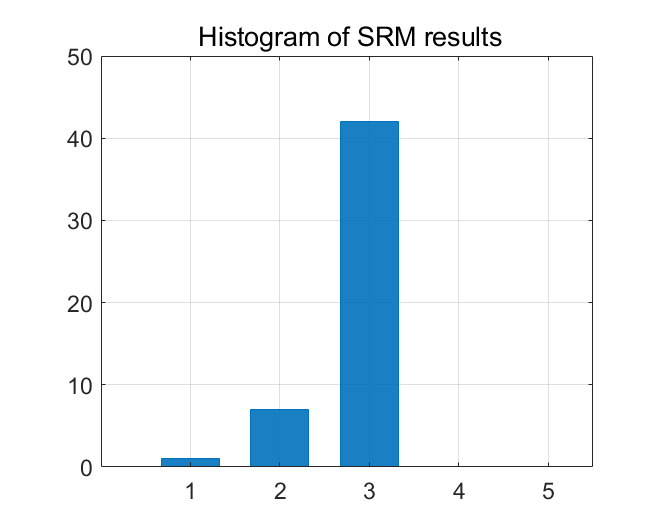} 
    }
  \caption{The SRM learning results with hypothesis function classes $\{\mathcal{F}_j\}_{j=1}^5$.} 
  \label{srm_linear_pic}
\end{figure}

Fig. \ref{srm_linear_pic} illustrates the mean results of $50$ experiments. The abscissa denotes the index of the function class $j$ from $1$ to $5$. We can find that as the model becomes complex, the empirical risk $\hat{\epsilon}_{\mathcal{T}}(r)$ decreases and after $j=3$, the real function model is contained in $\mathcal{F}_j$, thus the empirical risk changes little. Aligning with our intuitive, the penalty term for model complexity goes larger with $j$. Therefore, through adding these two terms together we obtain the optimal function class $j^*=3$ minimizing the structural risk in Fig. \ref{lam2}, which is consistent with \eqref{sim_lqr}. Since the SRM is data-dependent, different results may occur when noise exists. Fig. \ref{lam3} shows the statistics of $50$ trials and $j^*=1,2$ occurs with low probability. 

Notice that for a fixed function class, when the dataset size $M$ goes to infinity, the empirical risk will converge to the true error $\epsilon_{\mathcal{D}}(r)$, while the penalty term for model complexity will decrease as a speed of $\sqrt{M}$. This indicates different choices of $M$ lead to different optimal solutions. When $M$ is relatively small, the model is prone to overfitting. At the same time, the penalty term returns a high value, resulting in the SRM minimum being achieved with a simpler function class (small $j$), which effectively reduces the risk of overfitting. When the dataset is large, it is less likely to overfit, then our scheme will obtain more complex function class. 

\section{Conclusion}\label{conc_sec}
In this paper, an SRM scheme is provided for the model selection in IRL. For a series of hypothesis reward function classes, we utilize the policy gradient as the empirical risk and the upper bound on Rademacher complexity as the model penalty. Through minimizing the weighted sum of these two terms, we obtain the optimal reward identification, achieving a trade-off between the estimation error and generalization ability. 
We particularly analyze the linear weighted form in the simulation. This SRM scheme can also handle nonlinear hypothesis functions such as the kernel based representation.

\begin{appendices}
\section{Proof of Lemma \ref{rade_l}}\label{app_a}
According to Definition \ref{rade} and \eqref{erm_risk}, we have the empirical Rademacher complexity $\hat{\mathfrak{R}}_{\mathcal{T}}(\mathcal{L}^{irl})$ as
{\small{
\begin{align*}
&\mathbb{E}_{\sigma} \left[ \sup_{l^{irl}\in \mathcal{L}^{irl}} \frac{1}{M} \sum_{i=1}^M \sigma_i l^{irl}(s_0^i) \right] = \\
& \mathbb{E}_{\sigma} \left[ \sup_{r\in \mathcal{F}} \frac{1}{M} \sum_{i=1}^M \sigma_i l \left(\sum_{t=1}^{T} \nabla_{\theta} \ln{\pi^i_{t}} \sum_{k=t}^{T} \gamma^{k-t} r(s_{k}(s_0^i),a_{k}(s_0^i)) \right) \right]\\
& \leqslant L \mathbb{E}_{\sigma} \left[ \sup_{r\in \mathcal{F}} \frac{1}{M} \sum_{i=1}^M \sigma_i \sum_{t=1}^{T}  \sum_{k=t}^{T} \gamma^{k-t} \Vert \nabla_{\theta} \ln{\pi^i_{t}} \Vert r(s_{k}(s_0^i),a_{k}(s_0^i)) \right] \\
& \leqslant L \sum_{t=1}^{T} \sum_{k=t}^{T} \gamma^{k-t} \mathbb{E}_{\sigma} \left[ \sup_{r\in \mathcal{F}} \frac{1}{M} \sum_{i=1}^M \Vert\nabla_{\theta} \ln{\pi^i_{t}}\Vert \sigma_i r(s_{k}(s_0^i),a_{k}(s_0^i)) \right]\\
& \leqslant L \sum_{t=1}^{T} \sum_{k=t}^{T} \gamma^{k-t} \nabla_{\theta} \Tilde{\pi}_t \mathbb{E}_{\sigma} \left[ \sup_{r\in \mathcal{F}} \frac{1}{M} \sum_{i=1}^M  \sigma_i r(s_{k}(s_0^i),a_{k}(s_0^i)) \right]\\ &=  L \sum_{t=1}^{T} \sum_{k=t}^{T} \gamma^{k-t} \nabla_{\theta} \Tilde{\pi}_t \cdot \hat{\mathfrak{R}}_{\mathcal{T}_k}(\mathcal{F}),
\end{align*}
}}
where $\mathcal{T}_k = \{(s^i_k,a^i_k)\}_{i=1}^M$ and $\nabla_{\theta} \Tilde{\pi}_t = \max_{i} \Vert\nabla_{\theta} \ln{\pi^i_{t}}\Vert$.

\section{Proof of Theorem \ref{srm_the}}\label{app_b}
(i) Notice that the Rademacher complexity is defined on the bounded function class. We have the clipped version reward function $\Vert \bar{r} \Vert\leqslant B$. Since the loss function $l(\cdot)$ is a Lipschitz function with the Lipschitz constant $L$ and $l(0)=0$, we derive the bound on $\bar{\mathcal{L}}^{irl}$ as
{\small{
\begin{align*}
& \left| l \left(\sum_{t=1}^{T} \nabla_{\theta} \ln{\pi_{t}} \sum_{k=t}^{T} \gamma^{k-t} \bar{r}(s_k,a_k)) \right)\right| \\ &\leqslant L \Vert \sum_{t=1}^{T} \nabla_{\theta}\ln{\pi_{t}} \sum_{k=t}^{T} \gamma^{k-t} \bar{r}(s_k,a_k)) \Vert \\
& \leqslant L\sum_{t=1}^{T} \sum_{k=t}^{T} \Vert\nabla_{\theta}\ln{\pi_{t}} \cdot \gamma^{k-t}\Vert \cdot \Vert\bar{r}(s_k,a_k)\Vert \\
& \leqslant LB \sum_{t=1}^{T} \sum_{k=t}^{T} \gamma^{k-t}\Vert\nabla_{\theta}\ln{\pi_{t}}\Vert =: B_{\bar{\mathcal{L}}}.
\end{align*}
}}
Therefore, combining Theorem \ref{rade_def} and the bound on $\bar{\mathcal{L}}^{irl}_j$, for any $r \in \mathcal{F}_j$, we have
{\small{
\begin{align*}
&\epsilon_{\mathcal{D}}(\bar{r}) \leqslant \hat{\epsilon}_{\mathcal{T}}(\bar{r})+ 2 \hat{\mathfrak{R}}_{\mathcal{T}}(\bar{\mathcal{L}}^{irl}_j) + 3 B_{\bar{\mathcal{L}}} \sqrt{\frac{\log \frac{4}{\delta}}{2M}} \\
& \leqslant \hat{\epsilon}_{\mathcal{T}}(\bar{r}) + 2L \sum_{t=1}^{T} \sum_{k=t}^{T} \gamma^{k-t} \nabla_{\theta} \Tilde{\pi}_t \cdot \hat{\mathfrak{R}}_{\mathcal{T}_k}(\bar{\mathcal{F}}_j) + \\& 3 LB \left(\sum_{t=1}^{T}  \sum_{k=t}^{T} \gamma^{k-t} \Vert\nabla_{\theta}\ln{\pi_{t}}\Vert \right) \sqrt{\frac{\log \frac{4}{\delta}}{2M}} = \hat{\epsilon}_{\mathcal{T}}(\bar{r})+\\
& L \sum_{t=1}^{T} \sum_{k=t}^{T} \gamma^{k-t} \left(2 \nabla_{\theta} \Tilde{\pi}_t  \hat{\mathfrak{R}}_{\mathcal{T}_k}(\bar{\mathcal{F}}_j) + 3B \Vert\nabla_{\theta}\ln{\pi_{t}}\Vert \sqrt{\frac{\log \frac{4}{\delta}}{2 M}}\right).
\end{align*}
}}

(ii) To derive the SRM learning bound, we consider $\mathrm{P}(X_1+X_2 >\varepsilon) \leqslant \mathrm{P}(X_1 > \frac{\varepsilon}{2}) + \mathrm{P}(X_2 > \frac{\varepsilon}{2})$. Then we have
{\small{
\begin{equation}\nonumber
\begin{aligned}
& \mathrm{P}\left(\epsilon_{\mathcal{D}}(\bar{r}_{\mathcal{T}}^{SRM}) - \epsilon_{\mathcal{D}}(\bar{r}) - 4 \hat{\mathfrak{R}}_{\mathcal{T}}(\bar{\mathcal{L}}^{irl}_{j(r)}) > \varepsilon \right) \\
&\leqslant \mathrm{P} \left(\epsilon_{\mathcal{D}}(\bar{r}_{\mathcal{T}}^{SRM}) - \epsilon_{\mathcal{T}}(\bar{r}_{\mathcal{T}}^{SRM}) - 2 \hat{\mathfrak{R}}_{\mathcal{T}}(\bar{\mathcal{L}}^{irl}_{j({r}_{\mathcal{T}}^{SRM})}) > \frac{\varepsilon}{2}\right)  \\
&+ \mathrm{P} \left(\epsilon_{\mathcal{T}}(\bar{r}_{\mathcal{T}}^{SRM}) +2 \hat{\mathfrak{R}}_{\mathcal{T}}(\bar{\mathcal{L}}^{irl}_{j({r}_{\mathcal{T}}^{SRM})}) - \epsilon_{\mathcal{D}}(\bar{r}) -4 \hat{\mathfrak{R}}_{\mathcal{T}}(\bar{\mathcal{L}}^{irl}_{j(r)}) > \frac{\varepsilon}{2}\right) \\
& \leqslant \mathrm{P} \left( \sup_{r \in \mathcal{F}} (\epsilon_{\mathcal{D}}(\bar{r}) - \epsilon_{\mathcal{T}}(\bar{r}) - 2 \hat{\mathfrak{R}}_{\mathcal{T}}(\bar{\mathcal{L}}^{irl}_{j(r)})) > \frac{\varepsilon}{2}\right)  + \\
& ~~~~\mathrm{P} \left(\epsilon_{\mathcal{T}}(\bar{r}) - \epsilon_{\mathcal{D}}(\bar{r}) - 2 \hat{\mathfrak{R}}_{\mathcal{T}}(\bar{\mathcal{L}}^{irl}_{j(r)}) > \frac{\varepsilon}{2}\right),
\end{aligned}
\end{equation}
}}
where
{\small{
\begin{equation}\nonumber
\begin{aligned}
& \mathrm{P} ( \sup_{r \in \mathcal{F}} (\epsilon_{\mathcal{D}}(\bar{r}) - \epsilon_{\mathcal{T}}(\bar{r}) - 2 \hat{\mathfrak{R}}_{\mathcal{T}}(\bar{\mathcal{L}}^{irl}_{j(r)})) > \frac{\varepsilon}{2}) \\
&= \mathrm{P} ( \sup_{1\leqslant j \leqslant C} \sup_{r \in \mathcal{F}_j} (\epsilon_{\mathcal{D}}(\bar{r}) - \epsilon_{\mathcal{T}}(\bar{r}) - 2 \hat{\mathfrak{R}}_{\mathcal{T}}(\bar{\mathcal{L}}^{irl}_{j(r)})) > \frac{\varepsilon}{2}) \\
& \leqslant \sum_{j=1}^C \mathrm{P} ( \sup_{r \in \mathcal{F}_j} (\epsilon_{\mathcal{D}}(\bar{r}) - \epsilon_{\mathcal{T}}(\bar{r}) - 2 \hat{\mathfrak{R}}_{\mathcal{T}}(\bar{\mathcal{L}}^{irl}_{j(r)})) > \frac{\varepsilon}{2}) \\
& \leqslant 4C \exp\{-\frac{M \varepsilon^2}{18 B_{\bar{\mathcal{L}}}^2}\},
\end{aligned}
\end{equation}
}}
and
{\small{
\begin{equation}\nonumber
\mathrm{P} (\epsilon_{\mathcal{T}}(\bar{r}) - \epsilon_{\mathcal{D}}(\bar{r}) - 2 \hat{\mathfrak{R}}_{\mathcal{T}}(\bar{\mathcal{L}}^{irl}_{j(r)}) > \frac{\varepsilon}{2}) \leqslant 4 \exp\{-\frac{M \varepsilon^2}{18 B_{\bar{\mathcal{L}}}^2}\}.
\end{equation}
}}
Therefore, we derive
{\small{
\begin{equation}
\mathrm{P}(\epsilon_{\mathcal{D}}(\bar{r}_{\mathcal{T}}^{SRM}) - \epsilon_{\mathcal{D}}(\bar{r}) - 4 \hat{\mathfrak{R}}_{\mathcal{T}}(\bar{\mathcal{L}}^{irl}_{j(r)}) \!>\! \varepsilon) \leqslant 4 (C+1) \exp\{-\frac{M \varepsilon^2}{18 B_{\bar{\mathcal{L}}}^2}\}.
\end{equation}
}}
Set the right side equal to $\delta$, then (ii) has been proved.

\vspace{-5pt}

\section{Proof of Lemma \ref{linear_rade}}\label{app_c}
According to the Rademacher complexity defined in Definition \ref{rade}, we have
{\small{
\begin{align*}
& \hat{\mathfrak{R}}_{\mathcal{T}_k}(\mathcal{F}) = \mathbb{E}_{\sigma} \left[ \sup_{r\in \mathcal{F}} \frac{1}{M} \sum_{i=1}^M  \sigma_i r(s_{k}(s_0^i),a_{k}(s_0^i)) \right] \\ &= \mathbb{E}_{\sigma} \left[ \sup_{\Vert \omega_p \Vert \leqslant B_{\omega}} \frac{1}{M} \sum_{i=1}^M  \sigma_i \sum_{p=1}^q \omega^T_p \phi_p(s_{k}(s_0^i),a_{k}(s_0^i)) \right]\\
& 
\leqslant \frac{1}{M}\mathbb{E}_{\sigma} \left[ \sup_{\Vert \omega_p \Vert \leqslant B_{\omega}} \sum_{p=1}^q \Vert \omega^T_p \Vert \cdot \Vert \sum_{i=1}^M  \sigma_i \phi_p(s_{k}(s_0^i),a_{k}(s_0^i))\Vert \right]\\
& \leqslant \sum_{p=1}^q \frac{B_{\omega}}{M}\mathbb{E}_{\sigma}  \Vert \sum_{i=1}^M  \sigma_i \phi_p(s_{k}(s_0^i),a_{k}(s_0^i))\Vert  \\ & = \sum_{p=1}^q \frac{B_{\omega}}{M}\mathbb{E}_{\sigma} \sqrt{\sum_{m=1,n=1}^M  \sigma_m \sigma_n \phi^T_p(s^m_{k},a^m_{k}) \phi_p(s^n_{k},a^n_{k})}  \\
& \leqslant \sum_{p=1}^q \frac{B_{\omega}}{M} \sqrt{\sum_{m=1,n=1}^M  \phi^T_p(s^m_{k},a^m_{k}) \phi_p(s^n_{k},a^n_{k}) \mathbb{E}_{\sigma}\left[ \sigma_m \sigma_n  \right]} \\ & \leqslant \sum_{p=1}^q \frac{B_{\omega}}{M} \sqrt{\sum_{m=1,n=1}^M   \phi^T_p(s^m_{k},a^m_{k}) \phi_p(s^n_{k},a^n_{k})}\\
&\leqslant \sum_{p=1}^q \frac{B_{\omega}}{M} \sqrt{ M \Phi^2_p (k)} =  \frac{B_{\omega} }{\sqrt{M}} \sum_{p=1}^q \Phi_p(k),
\end{align*}
}}
where $\Phi_p(k)$ is the upper bound for $\Vert \phi_p(s,a) \Vert$ on $\mathcal{T}_k$.

\end{appendices}

\bibliographystyle{IEEEtran}
\bibliography{ref}

\end{document}